\documentclass{biometrika}

\usepackage{caption}
\usepackage{amsmath}
\makeatletter
\providecommand\@LN[1]{}  % 定义空命令，防止模板调用 \@LN 报错
\makeatother

\usepackage{tikz}
\usetikzlibrary{shapes.geometric}
\usepackage{graphicx}
\usepackage{amsmath, amsfonts, amssymb}
\usepackage{multirow}

\usepackage{longtable}
\usepackage{algorithm}
\usepackage{algorithmic} 
\newcommand{\indep}{\rotatebox[origin=c]{90}{$\models$}}
\DeclareRobustCommand\model{\mathrel{|}\joinrel\mkern-.5mu\mathrel{\neq}}
\newcommand{\ndep}{\rotatebox[origin=c]{90}{$\model$}}

\usepackage{booktabs}
\usepackage{needspace}
\usepackage[colorlinks=true,
linkcolor=blue,
citecolor=blue,
urlcolor=blue]{hyperref}

\usepackage{newtxtext}
\usepackage[subscriptcorrection]{newtxmath}

\graphicspath{{./art/}}



\begin{document}
	
	%\jname{None}
	\jname{arXiv Preprint}
	%% The year, volume, and number are determined on publication
	%\jyear{2025}
	%\jvol{103}
	%\jnum{1}
	%\cyear{2025}
	%% The \doi{...} and \accessdate commands are used by the production team
	%\doi{10.1093/biomet/asm023}
	%\accessdate{Advance Access publication on 31 July 2023}

	%\received{2 January 2017}
	%\revised{1 August 2023}
	
	%% The left and right page headers are defined here:
	\markboth{P. Heng et~al.}{style}

	\title{Structural Dimension Reduction in Bayesian Networks}
	
	\author{Pei Heng}
	\affil{KLAS and School of Mathematics and Statistics, Northeast Normal University,\\ Changchun, China
		\email{peiheng@nenu.edu.cn}}
		
	%\author{Shiyuan He, and Jianhua Guo}
	%\affil{School of Mathematics and Statistics, Beijing Technology and Business University,\\ Beijing, China
		%\email{heshiyuan@btbu.edu.cn; jhguo@btbu.edu.cn}}
		
	\author{Jianhua Guo}
	\affil{School of Mathematics and Statistics, Beijing Technology and Business University,\\ Beijing, China
		\email{jhguo@btbu.edu.cn}}
	
	\author{Yi Sun}
	\affil{College of Mathematics and System Science, Xinjiang University,\\ Urumqi, China \email{brian@xju.edu.cn}}

	\maketitle

\begin{abstract}
This work introduces a novel technique, named structural dimension reduction, to collapse a Bayesian network onto a minimum and localized one while ensuring that probabilistic inferences between the original and reduced networks remain consistent. To this end, we propose a new combinatorial structure in directed acyclic graphs called the directed convex hull, which has turned out to be equivalent to their minimum localized Bayesian networks. An efficient polynomial-time algorithm is devised to identify them by determining the unique directed convex hulls containing the variables of interest from the original networks. Experiments demonstrate that the proposed technique has high dimension reduction capability in real networks, and the efficiency of probabilistic inference based on directed convex hulls can be significantly improved compared with traditional methods such as variable elimination and belief propagation algorithms. The code of this study is open at \href{https://github.com/Balance-H/Algorithms}{https://github.com/Balance-H/Algorithms} and the proofs of the results in the main body are postponed to the appendix.
\end{abstract}


\renewcommand\thefootnote{\fnsymbol{footnote}}
\setcounter{footnote}{1}

\section{Introduction}\label{sec-1}	
Bayesian networks (BNs) integrate probability theory and graph theory to represent intricate relationships among interconnected variables within a system. They are an effective tool for modeling complex and uncertain systems, with a methodology that relies on directed acyclic graphs (DAGs). In these DAGs, vertices typically represent random variables, and directed edges signify dependencies among variables. Due to their interpretability and scalability, BNs have become popular graphical models and are increasingly applied to a wide range of tasks such as intelligence reasoning~\cite{koller2009probabilistic}, machine learning \cite{ji2019probabilistic}, decision analysis \cite{fenton2018}, and more.

Inference in BNs refers to deriving information about unknown variables based on known ones in the network. The core task of inference is to calculate the marginal distribution, which involves computing the probability distribution of specific variables of interest by summing over all possible values of the other variables. The application of inference in BNs is widely used in various fields, including but not limited to gene inference \cite{2016Inference}, healthcare \cite{KYRIMI2021102108}, finance \cite{sanford2015operational}, and educational assessment \cite{Culbertson2016BayesianNI}.

To perform marginal inference in BNs, techniques such as variational inference \cite{jordan1999introduction}, importance sampling \cite{pearl1988probabilistic}, variable elimination \cite{zhang1994asim}, belief propagation \cite{nath2010efficient}, and Markov Chain Monte Carlo (MCMC) methods \cite{neal1993probabilistic} are commonly employed. Most of these techniques can efficiently make inferences based on the observed data if the networks are small. Among these methodologies, variational inference, importance sampling, and MCMC are approximate inference methods. While these methods are useful when exact analytical solutions are unfeasible or computationally expensive, they still face challenges in terms of computational complexity, accuracy, and high variance when dealing with large-scale networks and high-dimensional data \cite{blei2017variational,robert1999monte}. On the other hand, variable elimination and belief propagation are exact inference techniques. Both provide systematic and efficient methods for calculating exact posterior probabilities and performing probabilistic inference, although their applicability depends on the structure and complexity of the graphical models under consideration \cite{zhang1994asim,nath2010efficient}.

In recent decades, various authors have studied marginalization in BNs to improve the efficiency of statistical inference in graphical models. For instance, to capture conditional independence relations in marginal distributions, they have developed several classes of graphical models, including the DAG model as a subclass. Examples include MC-graphs \cite{koster2002marginalizing}, ribbonless graphs \cite{sadeghi2013stable,sadeghi2016marginalization}, ancestral graphs \cite{richardson2002ancestral}, and acyclic mixed graphs \cite{evans2016graphs,forre2017markov}. These methods often require adding various types of edges to represent independence maps of marginal distributions. This operation alters the topology of these networks, leading to reparameterization and increasing model complexity and estimation variance. Consequently, this escalation in complexity and variance inevitably raises the computational demands and storage requirements during probabilistic inference.

During probabilistic queries or inference on variables of interest, numerous redundant variables frequently offer no meaningful information. These redundant variables significantly increase the complexity of inference calculations. This observation underscores the necessity of developing a structural dimension reduction mechanism, facilitating the transformation of a BN into a localized form while preserving the equivalence of inference outcomes between the original and reduced networks. It is widely recognized that barren vertices, which are not included in the query variables, can be safely removed from the graph without compromising the accuracy of the inference. Therefore, the initial step in dimensionality reduction involves eliminating these irrelevant vertices. Numerous researchers have introduced various reduction rules, leading to the development of distinct local models.

Darwiche \cite{darwiche2009} removed all outgoing edges from evidence vertices and reduced the corresponding conditional probability tables (CPTs), showing that inference in this sub-model is accurate. Lin and Druzdzel \cite{lin1997} proposed relevance reasoning, which collapses all vertices that are not on the path between the evidence vertices and the variable of interest and recalculates the CPTs for the variables whose parent vertices have changed. However, further reductions in these two sub-models are driven by the evidence variables, which implies that no additional variables can be collapsed when the evidence variables are empty. Furthermore, these methods require prior knowledge of the CPTs of the original model before reduction and necessitate recalculating the CPTs of certain vertices after reduction, which can be quite complex in high-dimensional cases. Finally, current structural learning algorithms can only learn a single graph from the equivalence class of DAGs. However, the ability of these methods to collapse may vary significantly across different equivalent DAGs. For example, in a directed path $L=v_0\rightarrow v_1\rightarrow,\dots,\rightarrow v_{n-1} \rightarrow v_n $ of length $n$, when the variables of interest are $v_{n-1}$ and $v_n$, and the evidence variables are empty, both Darwiche's method and Lin and Druzdzel's method cannot collapse any vertices. In the equivalent DAG $v_0\leftarrow v_1\leftarrow,\dots,\leftarrow v_{n-1} \leftarrow v_n $, however, both methods can collapse $n - 1$ vertices.

This idea of the above arguments is illustrated in the middle routine of Fig.~\ref{Fig3}, where the other two routines correspond to the aforementioned techniques, such as approximate or exact inference and marginalization operations, respectively.

\begin{figure*}[t]
	\centering
	\begin{tikzpicture}[scale=0.54]%[node distance = 1.5cm]
		\node at (-1.5,0){$P(x_V)$};
		\node at (5,2.5){{\footnotesize Approximate\ inference}};
		\node at (5,2){{\footnotesize or\ Exact \ inference }};
		\node at (5.2,-2.3){{\footnotesize Marginalization}};
		\node at (3.1,0.5){$R\subseteq H\subseteq V$};
		\node at (3.1,-0.5){$P(x_H)=P_{G_H}(x_H)$};
		
		\node at (7.8,0){$\hat{P}_{G_H}(x_H)$};
		\node at (12,0){$\hat{P}(x_R)$};			
		
		\draw[->,ultra thick] (7.6,2.2)--(11,0.5);
		\draw[->,ultra thick] (7.3,-2.2)--(11,-0.5);			
		\draw[->,ultra thick] (-0.3,0)--(6,0);
		\draw[->,ultra thick] (-0.5,0.5)--(2.4,2.2);
		\draw[->,ultra thick] (-0.5,-0.5)--(3.2,-2.2);			
		\draw[->,ultra thick] (9.4,0)--(11,0);			
	\end{tikzpicture}
	\caption{Techniques for inference of marginal probability, where $P(x_W)$ is a joint probability distribution over the random vector $x_W=\{x_w\}_{w\in W}$  in a BN and $\hat{P}(x_W)$  is its estimation.}%, while $P_{G}(X_W)$ is the joint probability distribution Markovian w.r.t. $G$.}
\label{Fig3}
\vspace{-5mm}
\end{figure*}
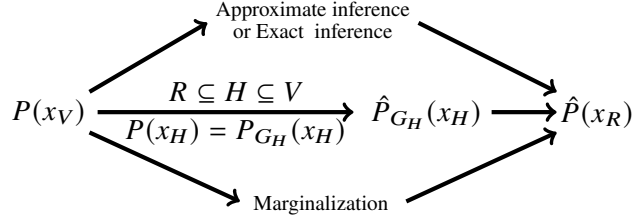

It is worthwhile to note that the proposed structural dimension reduction will benefit us in at least four aspects:
\begin{itemize}
	\item It enables direct calculations or the application of existing inference algorithms for probabilistic inference on smaller, localized BNs, thereby improving reasoning efficiency.
	\item In high-dimensional settings, collapsing irrelevant variables helps in handling small-sample data and missing data, thus avoiding the curse of dimensionality and improving inference accuracy.
	\item By collapsing onto a smaller, localized BN, computational efficiency is improved as fewer parameters need to be estimated.
	\item The dimensionality reduction capability, which is described in subsection \ref{sec-5.1} is not affected by structural learning algorithms, meaning it has the same reduction capability in equivalent DAGs.
\end{itemize}

In this paper, we revisit the structural dimensionality reduction of BNs from the perspective of collapsibility. Collapsibility means that the implied model for the family of marginal distributions over the subset of variables of interest is equivalent to the graphical model that is Markov relative to the induced subgraph by the subset \cite{frydenberg1990marginalization, wang2011finding}. As the foundation of structural dimension reduction, it essentially involves identifying the minimum collapsible set that includes the variables of interest. For recent work on collapsibility in BNs, readers can refer to \cite{kim2006note, xie2009collapsibility}. Additionally, several statisticians have developed fruitful results on collapsibility in undirected graphical models, such as \cite{asmussen1983, ducharme1986testing, frydenberg1990marginalization, liu2013collapsibility, wermuth1987parametric, whittaker2009graphical, whittemore1978collapsibility, madigan1990extension, wang2011finding}.

To achieve structural dimension reduction in BNs, we have devised an algorithm called the Close Minimal D-separator Absorption (CMDSA) algorithm, which has polynomial-time complexity for identifying the minimum collapsible set (i.e., the minimum and localized BN) containing variables of interest in a BN. This leads to more efficient probabilistic inference based on the localized BN.

Specifically, the theoretical contributions of this work can be summarized as follows:

\begin{itemize}
	%\item We propose the concept of a d-convex subgraph in directed graphs, which can be seen as a generalization of the convex subgraph defined in undirected graphs.
	\item We establish the equivalence between the minimum model collapsible set containing variables of interest and the directed convex hull including these relevant variables in a given directed graphical model.
	\item We discover and provide a specific connection between the directed convex hull and the minimal d-separation set.
	\item Based on this connection, we propose a polynomial-time complexity algorithm named CMDSA to find the minimum collapsible sets containing relevant variables.
\end{itemize}

The remaining structure of this paper is organized as follows. Section~\ref{sec-2} provides the necessary preliminary knowledge to make the paper self-contained. In Section~\ref{sec-3}, we propose the concept of a directed convex subgraph in DAGs and then detail its properties to devise an efficient algorithm for identifying it. Additionally, we show that the minimum collapsible set is equivalent to the directed convex hull under the assumption of faithfulness. In Section~\ref{sec-4}, we discuss the relationship between directed convex subgraphs and minimal d-separators, based on which we design the CMDSA algorithm for identifying the minimum directed convex subgraph containing variables of interest. In Section~\ref{sec-5}, we conduct experimental studies to demonstrate the validity and practicality of the proposed methodology. Finally, the conclusion and discussion are summarized in Section~\ref{sec-6}.

\section{Preliminaries}\label{sec-2}

A DAG, denoted by $G=(V, E)$, consists of a finite set of vertices $V$, corresponding to random variables $x_V=\{x_v\}_{v\in V}$ whose values are from the sample space $\mathcal{X}_V=\bigotimes_{v\in V}\mathcal{X}_v$, and a set of directed edges $E$ that captures dependencies among distinct variables. A directed edge from vertex $u$ to $v$ in $G$ is denoted by $(u, v)$. If $(u, v)\in E$, we say that $u$ is the parent of $v$, and $v$ is the child of $u$. Let $pa_G(v)$ (resp. $ch_G(v)$) represent the set of all parent (resp. child) vertices of $v$ in $G$. We use $u\overset{G}\sim v$ to denote the presence of an edge between $u$ and $v$ in graph $G$, indicating that they are adjacent in $G$.

A path connecting $u,v$ in $G$, denoted by $l_{uv}$, is composed of a series of distinct vertices $(u = u_0, u_1,\dots,u_k = v)$ for which $u_{i-1}\overset{G}\sim u_{i}$ for $i\in [k]:=\{1,\ldots,k\}$. The set of all vertices on $l_{uv}$ is written as $V(l_{uv}) = \{u_0, u_1,\dots,u_k \}$,  and the intermediate vertices are $V^o(l_{uv})=V(l_{uv})\setminus\{u,v\}$. Let $V^c(l_{uv})= \{w:\ \exists\ w_1\rightarrow w\leftarrow w_2\ \subseteq\ l_{uv}\}$ and $V^{nc}(l_{uv})= V^o(l_{uv})\backslash V^c(l_{uv})$ denote the collider and non-collider vertices on the path  
$l_{uv}$, respectively. For $k\geq 1$, if on the path $l_{uv}$, $(u_{i-1}, u_{i})\in E$ for all $i\in [k]$, we say that $l_{uv}$ is a directed path from $u$ to $v$, and $u$ (resp. $v$) is  an ancestor (resp. descendant) of $v$ (resp, $u$). We use the symbol $an_G(u)$ (resp. $de_G(u))$ to represent the set of all ancestors (resp. descendants) of $u$ in $G$. For a subset $A\subseteq V$, the ancestor closure set of $A$ in $G$ is written as $An_G(A) = \cup_{v\in A}an_G(v)\cup A$, and  its ancestor set is defined by $an_G(A) = An_G(A)\backslash A$. Let $mb_G(A) = \cup_{v\in A}\left(pa_{G}(v)\cup ch_{G}(v) \cup pa_G(ch_G(v))\right)\backslash A$ denote the Markov boundary of $A$. 
Additionally, the induced subgraph of $G$ on $A$ is defined as $G_A=(A,E_A)$, where $E_A = E\cap (A\times A)$.  %Here $mb_G(A)$ is nothing but the so-called Markov boundary of $A$ in $G$.

% With the graphical terminology in hand, we now introduce a new concept called d-connectedness, which generalizes the connectedness in graph theory.
\begin{definition}[D-connectedness]\label{d-connectedness}
	Given a DAG $G=(V,E)$ and $R\subseteq V$. For $M\subseteq V\backslash R$, we say that the induced subgraph $G_M$ is d-connected in $G$ (With respect to) w.r.t. $R$ if for any two distinct vertices $u,v\in M$ there exists a path $l_{uv}$ in $G$ such that $V^o(l_{uv})\cap R \subseteq V^c(l_{uv})$. Further, if $G_M$ is d-connected w.r.t. $R$ and for any $M^\prime$ satisfying $ M \subsetneq M^\prime \subseteq V\backslash R$, $G_{M^\prime}$ is not d-connected in $G$ w.r.t. $R$, we say that it is maximal w.r.t. $R$ and also say that $G_M$ or $M$ is a d-connected component in $G$ w.r.t. $R$.
\end{definition}

The moral graph of $G$ is an undirected graph, denoted by $G^m = (V,E^m)$, in which $E^m = \{\langle u,v\rangle:\ u\overset{G}\sim v\ \text{or there is a vertex}\ w\in V\  \text{such that } u,v\in pa_G(w)\}$.  It is not difficult to show that a d-connected component in $G$ must be connected in its moral graph. For ease of understanding, here is an example.

\begin{example}
	Let $G=(V,E)$ be the DAG shown in Figure~\ref{fig-22} (A) and $R=\{a,i\}$. Then the subgraph $G_{\{b,c,d,e,f,g\}}$ shown in  Figure~\ref{fig-22} (B) is d-connected w.r.t. $R$. 
	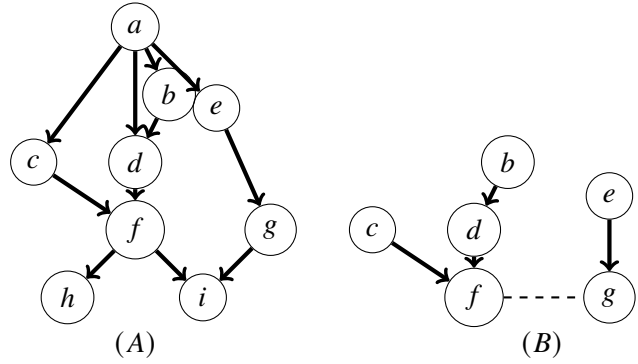
\begin{figure*}[t]
		\centering
		\begin{tikzpicture}[scale=0.9]
			\node[shape=circle, draw=black] (1) at (0,0){$a$};
			\node[shape=circle, draw=black] (2) at (0.5,-1){$b$};
			\node[shape=circle, draw=black] (3) at (-1.5,-2){$c$};
			\node[shape=circle, draw=black] (4) at (0,-2){$d$};
			\node[shape=circle, draw=black] (5) at (1.2,-1.2){$e$};
			\node[shape=circle, draw=black] (6) at (0,-3){$f$};
			\node[shape=circle, draw=black] (7) at (2,-3){$g$};
			\node[shape=circle, draw=black] (8) at (-1,-4){$h$};
			\node[shape=circle, draw=black] (9) at (1,-4){$i$};
			\draw[->,ultra thick] (1)--(2);
			\draw[->,ultra thick] (1)--(3);
			\draw[->,ultra thick] (1)--(4);
			\draw[->,ultra thick] (1)--(5);
			\draw[->,ultra thick] (2)--(4);
			\draw[->,ultra thick] (5)--(7);
			\draw[->,ultra thick] (4)--(6);	
			\draw[->,ultra thick] (3)--(6);	
			\draw[->,ultra thick] (6)--(8);	
			\draw[->,ultra thick] (6)--(9);	
			\draw[->,ultra thick] (7)--(9);	
			\node at (0,-4.7){$(A)$};
		\end{tikzpicture}
		\hspace{0.5cm}
		\begin{tikzpicture}[scale=0.9]
		\node[shape=circle, draw=black] (2) at (0.5,-1){$b$};
		\node[shape=circle, draw=black] (3) at (-1.5,-2){$c$};
		\node[shape=circle, draw=black] (4) at (0,-2){$d$};
		\node[shape=circle, draw=black] (5) at (2,-1.5){$e$};
		\node[shape=circle, draw=black] (6) at (0,-3){$f$};
		\node[shape=circle, draw=black] (7) at (2,-3){$g$};
		\draw[->,ultra thick] (2)--(4);
		\draw[->,ultra thick] (5)--(7);
		\draw[->,ultra thick] (4)--(6);	
		\draw[->,ultra thick] (3)--(6);	
		\draw[-,dashed,thick] (6)--(7);
		\node at (1,-3.7){$(B)$};			
		\end{tikzpicture}
		\caption{(A) A DAG $G=(V,E)$ with the set of interest $R=\{a,i\}$. (B) The d-connected subgraph $G_{\{b,c,d,e,f,g\}}$ w.r.t. $R$. }
		\label{fig-22}
	\end{figure*}
\end{example}

We say that a path $l_{uv}$ in $G$ is blocked by a subset $S\subseteq V\backslash \{u,v\}$ if one of the following two conditions is satisfied: (1) $S\cap V^{nc}(l_{uv}) \neq \emptyset$; (2) $\exists \ w\in V^c(l_{uv})$ such that $S\cap (de_G(w)\cup \{w\}) = \emptyset$. For pairwise disjoint subsets $X, Y, Z\subseteq V$, if all paths from $X$ to $Y$ are blocked by $Z$, we say that $X$ is d-separated from $Y$ by $Z$ in $G$ and also say that $Z$ is an XY-d-separator in $G$, denoted by $X \indep Y\ |\ Z [G]$. Otherwise, we say that $X$ and $Y$ are not d-separated given $Z$, rewritten as $X \ndep Y\ |\ Z [G]$.   
Under the criterion of d-separation, the graph $G$ induces an independence model, denoted by $I(G)$,  i.e., $I(G) = \{\langle X,Y|Z\rangle:\ X  \indep Y\ |\ Z [G] \text{~with pairwise disjoint subsets~} X,Y,Z\subseteq V\}.$
For a vertex subset $A\subseteq V$, $I(G)$ has a natural projection on $A$, denoted as $I(G)_A $, which is defined as follows: $I(G)_A  = \{\langle X,Y|Z\rangle:\ X \indep Y\ |\ Z [G] \text{ with pairwise disjoint subsets }X,Y,Z\subseteq A\}$.  Generally, $I(G)_A\neq I(G_A)$ but it is easy to show that $I(G)_A\subseteq I(G_A)$.% If $I(G)_A=I(G_A)$, we say that $I(G)$ is conditional independence collapsed onto $A$.

A BN, denoted by $\mathcal{B} = (G, \mathcal{P})$, consists of a DAG $G=(V,E)$ and a probability distribution family $\mathcal{P}:=\mathcal{P}(G)$ such that $P$ is Markovian w.r.t. $G$ for all $P\in \mathcal{P}$, i.e., $P(x_V) = \prod_{v\in V}P(x_v|x_{pa_G(v)})$, where $P(x_v|x_{pa_G(v)})$ is the conditional probability distribution of $x_v$ given $x_{pa_G(v)}$. We say that $G$ is faithful to $\mathcal{P}(G)$ if there is a distribution $Q\in \mathcal{P}(G)$ such that $I(G)=I(Q)$ where $I(Q)= \{\langle A,B|C\rangle:\ x_A  \indep x_B\ |\ x_C [Q] \text{~with pairwise disjoint subsets~} A,B,C\subseteq V\}$. Here
$x_A  \indep x_B\ |\ x_C [Q]$ represents that $x_A$ is independent of $x_B$ given $x_C$ in $Q$. 

Similar to the independence model induced by $G$, the distribution family $\mathcal{P}(G)$ has a projection onto a subset $A\subseteq V$ as well, which is given by 
$\mathcal{P}(G)_A = \left\{P_A:\ P_A = \int_{\mathcal{X}_{V\setminus A}}dP(x_V)\ \text{for}\ P\in \mathcal{P}(G)\right\}$, where $P_A$ is the marginal distribution of $P$ over $x_A$. Generally, $\mathcal{P}(G)_A\neq \mathcal{P}(G_A)$. Given a BN  $\mathcal{B} = (G, \mathcal{P})$ with $G=(V, E)$ and $A\subseteq V$,  we now describe two kinds of collapsibility in BNs that are useful in this work:

\begin{itemize}
	\item $\mathcal{B}$ is conditional independence collapsible (CI-collapsible) onto $A$ if $I(G)_A=I(G_A)$;
	\item $\mathcal{B}$ is model collapsible (M-collapsible) onto $A$ if $\mathcal{P}(G)_A=\mathcal{P}(G_A)$.
\end{itemize}
Under faithfulness, it has been proved that CI-collapsibility is equivalent to M-collapsibility \cite{xie2009collapsibility}.

\section{Directed convex subgraph with its properties}\label{sec-3}

This section extends the concept of convex subgraphs \cite{diestel1990graph, wang2011finding} in undirected graphs to directed convex subgraphs in DAGs. To begin with, we emphasize the concept of an inducing path, as proposed in \citet{verma2022equivalence}.

\begin{definition}[Inducing path]\label{def-1}
	For two distinct non-adjacent vertices $u,v\in R\subseteq V$ in $G$, we say that a path $l_{uv}$ connecting $u,v$ in $G$ is an inducing path of $R$ if  $R \cap V^{o}(l_{u v}) \subseteq V^{c}(l_{u v}) \subseteq an_G(\{u, v\}).$ In this case, we refer to $\{u,v\}$ as an \textbf{information pair} (abbreviated as i-pair) of $R$.
\end{definition}

We represent the set of all i-pairs of $R$ in $G$ as $\mathcal{I}_G(R)=\{\{u,v\}:\{u,v\} \text{~is~an~} \text{i-pair~of~} R \text{~in~}G\}.$ For example, in Figure \ref{fig-22}, 
$I_G(R) = \{\{a,i\}\}$, since the path $l_{ai} (=a\rightarrow e \rightarrow g\rightarrow i)$ is an induced path of $R=\{a,i\}$. When performing dimensionality reduction, it is crucial to preserve the inducing paths of certain connection information pairs; otherwise, incorrect conditional independence relationships may emerge. However, efficiently identifying inducing paths in high-dimensional graphs remains a significant challenge. The following result provides us with an efficient way to identify information pairs.

\begin{lemma}\label{thm-3.2}
	Let $G=(V,E)$ be a DAG, $R =V\backslash M$ for $M\subseteq V$ and $u,v\in R$ be two non-adjacent vertices. Then the following statements are equivalent:
	\begin{itemize}
		\item[(i)] $\{u,v\}\in \mathcal{I}_G(R)$;
		\item[(ii)] There is a path $\rho_{uv}$ connecting $u,v$ in $(G_{An_G(\{u,v\})})^m$ such that $V^o(\rho_{uv})\subseteq M$;
		\item[(iii)] $u \ndep v\ |\ an_G(\{u,v\})\backslash M[G] $;
		\item[(iv)]  $u  \ndep v\ |\ Z [G]$, for any $Z\subseteq R\backslash \{u,v\}$.
	\end{itemize}
\end{lemma}

Lemma~\ref{thm-3.2} provides several equivalent conditions for identifying i-pairs. This lemma provides a straightforward method for recognizing i-pairs, which are essential for designing algorithms for structural dimension reduction.

\begin{definition}[D-convexity]\label{def-d-convexity}
	Let $G=(V,E)$ be a DAG and $H\subseteq V$. We say that the induced subgraph $G_H$ is \textbf{directed convex} (or d-convex for short) if there is no inducing path of $H$ in $G$.
\end{definition}

\begin{theorem}\label{thm-3.5}
	Let $G=(V,E)$ be a DAG and $R\subseteq V$. Then the following statements are equivalent:
	\begin{itemize}
		\item[(i)] $G_R$ is a d-convex subgraph of $G$;
		\item[(ii)] $I(G)_R=I(G_R)$.
	\end{itemize}
\end{theorem}

Theorem \ref{thm-3.5} provides an equivalent condition for the d-convex subgraph, which is the basis for finding the localized BN, including relevant variables. For a subset $A\subseteq V$, we say it is a d-convex subset if the induced subgraph $G_A$ is d-convex. The following result demonstrates that d-convex subgraphs have the closure property for the set intersection operation.

\begin{theorem}[Closure property]\label{thm-3.7}
	Let $G=(V, E)$ be a DAG and $\Lambda $ an index set. For all $\lambda \in \Lambda$, if $H_{\lambda}$ is d-convex in $G$, then $\cap_{\lambda \in \Lambda} H_{\lambda}$ is d-convex in $G$.
\end{theorem}

Theorem~\ref{thm-3.7}
guarantees that the minimal d-convex subgraph containing $R$  is uniquely existent. We refer to the unique minimal d-convex subgraph containing $R$ as the  \textbf{d-convex hull} of $R$ in $G$. 

\begin{definition}[Minimum collapsible set]\label{def:mcs}
	Let $\mathcal{B}=(G,\mathcal{P}(G))$ be a BN and $R\subseteq V$. We say that the subset $H$ containing $R$ is the minimum collapsible set if it satisfies
	\begin{itemize}
		\item[(i)] $\mathcal{P}(G)_H=\mathcal{P}(G_H)$; and
		\item[(ii)] $\mathcal{P}(G)_{H'}\neq\mathcal{P}(G_{H'})$ for any subset $H'$ satisfying $R\subseteq H' \subsetneq H$.
	\end{itemize}    
\end{definition}

Given a BN $\mathcal{B} = (G, \mathcal{P}(G))$ and a subset $R$ of interest in probabilistic inference, Definition~\ref{def:mcs} suggests that inferences can be made on the variables in $R$ based on the localized BN $\mathcal{B}_H = (G_H, \mathcal{P}(G_H))$ over the minimum collapsible set $H$ that contains $R$. Essentially, the minimum collapsible set is the smallest set onto which a  BN can be collapsed, to obtain a localized network without affecting the inference results about variables in $R$.

\begin{theorem}\label{thm:m-coll=d-convex}
	Given a BN $\mathcal{B}=(G,\mathcal{P}(G))$ with $G=(V,E)$ and $H\subseteq V$. Under faithfulness, we have that $\mathcal{B}$ is M-collapsible onto $H$ if and only if $G_H$ is d-convex in $G$.    
\end{theorem}

\begin{proof}
	This follows easily from Theorem~\ref{thm-3.5} and Lemma~\ref{lem:xie-geng} in Appendix \ref{new_N.0}.
\end{proof}

Theorem~\ref{thm:m-coll=d-convex} implies the following important corollary, which transforms a difficult problem in statistics into a simple one in graph theory.

\begin{corollary}\label{cor-new}
	Given a BN $\mathcal{B}=(G,\mathcal{P}(G))$ with $G=(V,E)$ and $R\subseteq H\subseteq V$. Then, $H$ is the minimum collapsible set containing $R$ if and only if $G_H$ is the d-convex hull of $R$. 
\end{corollary}

The parameterization of the d-convex hull can be approached in two ways. The first method involves inference based on a known global CPT. When the parent node of a variable $u$ changes, $u$ is considered as the target node, and its parent node(s) in the d-convex hull (which may be empty) serve as conditions. The conditional probability distribution of $u$ is then recalculated using standard inference algorithms. Although the computation of the updated conditional distribution can be complex in certain cases, this method significantly accelerates subsequent inference processes. Therefore, caching frequently used local CPTs can improve inference efficiency. The second method involves reparameterizing the d-convex hull using corresponding local data and performing the inference, with the correctness of this approach stemming from Corollary \ref{cor-new}. This method does not require precomputing or storing the original model's CPT, nor does it require the recalculation of certain variables' CPTs. As noted by Xie and Geng \cite{xie2009collapsibility}, the second parameterization method ensures that the values of the two maximum likelihood estimates are asymptotically equivalent. Therefore, when only a few variables are of interest, preserving a smaller convex set still allows for efficient inference, making this method both more convenient and cost-effective. In summary, when the model's dimensionality is moderate and the CPT is known or precise inference is required, the first method is recommended. Otherwise, the second method is suggested.

\section{An efficient algorithm for d-convex hulls}\label{sec-4}

To introduce our proposed algorithm, we first explore the relationship between the d-convex subgraph and the minimal d-separator. For two non-adjacent vertices $u$ and $v$, a $uv$-d-separator $S$ is considered to be minimal if no proper subset of $S$ d-separates $u$ from $v$ in $G$. In what follows, we establish the theoretical foundations for the proposed algorithm. 

\begin{theorem}\label{thm-4.3}
	Let $G=(V,E)$ be a DAG and $H\subseteq V$. The subgraph $G_H$ is d-convex if and only if for any pair of non-adjacent vertices $u,v\in H$, all minimal $uv$-d-separators $S$ are contained in $H$.
\end{theorem}

Theorem~\ref{thm-4.3} states that one can find d-convex hulls by absorbing minimal d-separators for i-pairs. Here, we present a readily identifiable minimal d-separator, which helps the proposed algorithm identify the d-convex hull more efficiently.

\begin{definition}[Close minimal d-separator]\label{def:close-d-separator}
	For any non-adjacent vertices $u,v$ in a DAG $G=(V,E)$, let $S\subseteq V$ be a minimal $uv$-d-separator. We say that $S$ is close to $u$ if $S\subseteq mb_G(u)$.
\end{definition}

\begin{theorem}\label{theo-1}
	Let $G=(V, E)$ be a DAG and $u,v\subseteq V$ two non-adjacent vertices. Then $mb_{G_{An_G(\{u,v\})}}(u)$ contains the unique minimal $uv$-d-separator that is close to $u$.
\end{theorem}

Based on Theorem~\ref{theo-1}, an algorithm is proposed to identify the close minimal d-separator between two non-adjacent vertices in 
$G$; refer to Algorithm \ref{alg-1} in Appendix \ref{new_N.0}. Using the FCMDS algorithm, we can devise an algorithm for identifying the d-convex hull containing the variable of interest by absorbing close minimal d-separators. We call the proposed algorithm the Close Minimal D-separator Absorption Algorithm, abbreviated as CMDSA; see Algorithm~\ref{alg-2}.

\begin{algorithm}
	\caption{Close Minimal D-separator Absorption Algorithm (CMDSA)}
	\label{alg-2}
	\begin{algorithmic}
		\STATE {\bfseries Input:} A DAG $G=(V,E)$  and a subset $R\subseteq V$.
		\STATE {\bfseries Output:}  The d-convex hull $G_H$ containing $R$.
		\STATE Initialize $H:=R,\mathcal{Q}: = \emptyset$
		\REPEAT%  EPEAT
		\FOR{each d-connected component $G_M$ in $G_{An_G(R)}$ w.r.t. $H$}
		\STATE Update $\mathcal{Q}$ to $\mathcal{I}_{G_{M\cup H}}(H)$
		\FOR{all vertex pairs $\{u,v\}\in \mathcal{Q}$}
		\STATE Call FCMDS to obtain close minimal $uv$-d-separators $S_u$ and $S_v$ in $G$, which are close to $u$ and $v$, respectively
		\STATE Update $H$ to $H\cup S_u\cup S_v$
		%\STATE break;
		\ENDFOR
		\ENDFOR
		\UNTIL{$H$ contains no i-pairs}
		\STATE {\bfseries Return:} The induced subgraph $G_H$ of $G$ on $H$
	\end{algorithmic}
\end{algorithm}

During the implementation of CMDSA, the first step involves reducing the search space to the ancestor closure set $An_G(R)$ for a subset $R\subseteq V$, as $An_G(R)$ is a d-convex subset containing $R$. If $R$ is not d-convex, then the set of i-pairs $\mathcal{I}_G(R)\neq\emptyset$. For any $\{u,v\}\in \mathcal{I}_G(R)$,
Lemma \ref{new-lemma-1} (see Appendix \ref{new_N.0}) ensures that there exists a d-connected component $G_M$  within $G_{An_G(R)}$ such that $\{u,v\}\in \mathcal{I}_{G_{M\cup R}}(R)$. This step further narrows the search space for the i-pairs.

According to Theorem \ref{thm-4.3}, the d-convex hull 
$H$ containing $R$ must include all minimal $uv$-d-separators. To implement this, we incorporate the close minimal $uv$-d-separators into $R$, ensuring that 
$\{u,v\}$ is no longer an i-pair in the updated set. It is important to note that in CMDSA, as minimal d-separators are added to $H$, new i-pairs may be generated. Therefore, to ensure that all necessary i-pairs are destroyed, the algorithm requires continuous iterations until the updated d-convex hull $H$ no longer contains any i-pairs.

\begin{theorem}\label{thm-4.33}
	The induced subgraph $G_H\subseteq G$ obtained by CMDSA is the d-convex hull containing $R$.
\end{theorem}

\begin{proposition}\label{prop-5.8}
	The algorithm CMDSA has the complexity at most $O(|V|(|V|+|E|))$.
\end{proposition}

To illustrate Algorithm~\ref{alg-2}, here we give an example.

\begin{example}
	Let $\mathcal{B}=(G,\mathcal{P})$ be the BN with its network structure shown in Fig.\ref{fig-new}(A). Assuming we are interested in the posterior probability $P(h|d)$, then $R = \{d,h\}$ is the set of variables of interest. We use CMDSA to find the d-convex hull containing the variable set $R$.
	
	In the first cycle, CMDSA initializes $H$ as $R$. Notably, $G_{An_G(R)}$ has only one connected component $G_{V\backslash \{d,h\}}$ w.r.t. $H$, and we have $\mathcal{Q}=\{\{d,h\}\}$. In the subgraph $G_{An_G(\{d,h\})}$, we observe that the close minimal $dh$-d-separator is $\{c\}$. Thus, we update $H$ to $H \cup \{c\} = \{c,d,h\}$.
	
	In the second cycle, we find $\mathcal{Q}=\emptyset$. The algorithm finally outputs the induced subgraph $G_H$, as shown in Fig.\ref{fig-new}(B). The local model established by Darwiche \cite{darwiche2009} is Fig.\ref{fig-new}(A) itself, meaning that it cannot simplify any nodes in this case. The local model created by Lin and  Druzdzel \cite{lin1997} is depicted in Fig.\ref{fig-new}(C).

	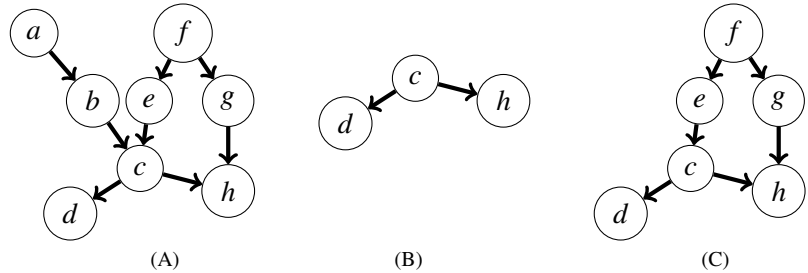
\begin{figure*}[t]
		%\vskip -0.1in
		\begin{center}
			\begin{tikzpicture}[scale=0.6]
				\node[shape=circle, draw=black] (1) at (-1.8,-1.5){$a$};
				\node[shape=circle, draw=black] (2) at (-0.5,-3){$b$};
				\node[shape=circle, draw=black] (3) at (1.5,-1.5){$f$};
				\node[shape=circle, draw=black] (4) at (0.75,-3){$e$};
				\node[shape=circle, draw=black] (5) at (2.5,-3){$g$};
				\node[shape=circle, draw=black] (6) at (0.55,-4.5){$c$};
				\node[shape=circle, draw=black] (7) at (2.5,-5){$h$};
				\node[shape=circle, draw=black] (8) at (-1,-5.5){$d$};
				\draw[->,ultra thick] (1)--(2);
				\draw[->,ultra thick] (2)--(6);
				\draw[->,ultra thick] (3)--(4);
				\draw[->,ultra thick] (3)--(5);
				\draw[->,ultra thick] (4)--(6);
				\draw[->,ultra thick] (6)--(8);	
				\draw[->,ultra thick] (6)--(7);	
				\draw[->,ultra thick] (5)--(7);	
				\node at (1.1,-6.55){\footnotesize (A)};
			\end{tikzpicture}
			\hspace{0.65cm}
			%\vspace{0.3cm}
			\begin{tikzpicture}[scale=0.6]
				\node[shape=circle, draw=black] (6) at (0.55,-2.5){$c$};
				\node[shape=circle, draw=black] (7) at (2.5,-3){$h$};
				\node[shape=circle, draw=black] (8) at (-1,-3.5){$d$};
				\draw[->,ultra thick] (6)--(7);
				\draw[->,ultra thick] (6)--(8);	
				\node at (0.4,-6.55){\footnotesize (B)};
			\end{tikzpicture}
			\hspace{0.65cm}
			%\vspace{0.3cm}
			\begin{tikzpicture}[scale=0.6]
				\node[shape=circle, draw=black] (3) at (1.5,-1.5){$f$};
				\node[shape=circle, draw=black] (4) at (0.75,-3){$e$};
				\node[shape=circle, draw=black] (5) at (2.5,-3){$g$};
				\node[shape=circle, draw=black] (6) at (0.55,-4.5){$c$};
				\node[shape=circle, draw=black] (7) at (2.5,-5){$h$};
				\node[shape=circle, draw=black] (8) at (-1,-5.5){$d$};
				\draw[->,ultra thick] (3)--(4);
				\draw[->,ultra thick] (3)--(5);
				\draw[->,ultra thick] (4)--(6);
				\draw[->,ultra thick] (6)--(8);	
				\draw[->,ultra thick] (6)--(7);	
				\draw[->,ultra thick] (5)--(7);	
				\node at (1.1,-6.55){\footnotesize (C)};
			\end{tikzpicture}
			\caption{(A) The Asia network $G$; (B) The d-convex hull containing the variable set $R=\{d,h\}$;  (C) Lin and  Druzdzel's method.}\label{fig-new}
			\end{center}
			\vskip -0.2in
		\end{figure*}
\end{example}

In Fig.~\ref{fig-new}(B), the algorithm needs to compute the prior distribution of variable $c$ and the conditional distribution $P(h|c)$, and based on this, construct the minimal collapsible model for inference. In contrast, in Fig.~\ref{fig-new}(C), the algorithm is required to recompute the conditional distribution $P(c|e)$. It is noteworthy that the d-convex hull model is the minimal local model, as removing any further variables would violate conditional independence.

\section{Experimental studies}\label{sec-5} 

This section details the dimension reduction capability of the proposed method and evaluates the accuracy and efficiency of probabilistic inference based on structural dimension reduction. The experiments were conducted on a 64-bit Windows operating system with a base frequency of 3.2 GHz, equipped with 16 GB of RAM. The implementation was done using the Python programming language.

\subsection{Dimension reduction capability}\label{sec-5.1}

We define the structural dimensionality reduction capability as $$\mathrm{DRC}(R)=1-\frac{|\mathrm{CH_G(R)}|}{n},$$ where $\mathrm{CH_G(R)}$ is the d-convex hull of $R$ in $G$. This definition reflects the extent to which irrelevant variables can be collapsed during the process of probabilistic inference. Experiments were conducted on two types of DAGs: one type was randomly generated, and the other type was sourced from the benchmark DAGs in the Bnlearn library (\href{https://www.bnlearn.com/bnrepository/}{www.bnlearn.com/bnrepository}).

The simulation process is described as follows:

\begin{itemize}
	\item For network sizes of $n=$200, 400, 600, and 800, with an edge connection probability of 0.01, 50 networks were generated for each size.
	\item Within each network, 20 sets of vertices of interest, $R$, were randomly selected, where the size of $R$ ranged from 2 to 8. The proposed CMDSA was then applied to identify d-convex hulls containing $R$.
	\item The average dimension reduction capability of the CMDSA was computed for networks of each scale.
\end{itemize}

\begin{center}
\begin{table*}[!t]
	\caption{DRC on randomly generated DAGs}
	\label{tab-2}
	\begin{center}	
		\begin{tabular*}{\textwidth}{@{\extracolsep\fill}ccccc@{}}
			\hline
			\textbf{n}  & 200     & 400     & 600     & 800     \\ \hline
			\textbf{DRC} & 97.58\% & 73.60\% & 53.46\% & 36.70\% \\ \hline
		\end{tabular*}%		
	\end{center}
\end{table*}
\end{center}

\begin{center}
\begin{table*}[!t]
	\caption{DRC on real benchmark DAGs}
	\label{tab-3}
	\begin{center}
	\begin{tabular*}{\textwidth}{@{\extracolsep\fill}cccccccc@{}}
		\hline
		\textbf{Network} & \textbf{V}   & \textbf{|E|} & \textbf{DRC} & \textbf{Network} & \textbf{V}  & \textbf{|E|} & \textbf{DRC} \\ \hline
		Alarm & 37   & 46  & 62.70\% & Diabetes  & 413  & 602  & 55.39\%\\ 
		Hepar2 & 70   & 123  & 75.14\% & Link   & 724  & 1125 & 95.43\%\\ 
		Andes  & 223  & 338  & 70.91\% & Munin2 & 1003 & 1124 & 98.66\%\\ \hline
	\end{tabular*}%
\end{center}
\end{table*}
\end{center}

The simulation results are shown in Table \ref{tab-2} and Table \ref{tab-3}, from which we observe the following:

\begin{itemize}
	\item The DRC of CMDSA gradually decreases as network size increases in the case where DAGs are randomly generated. This is because, under a given edge probability, as the number of network vertices increases, the network becomes more complex, leading to larger d-convex hull sizes that include $R$.
	\item The DRC of CMDSA performs excellently in the case where DAGs are from practical scenarios. This is because real DAGs are typically sparse and exhibit inherent logical relationships among variables, resulting in significantly better dimensionality reduction compared to random graphs.
\end{itemize}

\subsection{Efficiency and  accuracy evaluation}

We compare the efficiency and accuracy of probabilistic inference based on the proposed structural dimension reduction method with well-known probabilistic inference techniques, such as variable elimination (VE), lazy propagation (LP) \cite{Madsen1999}, and improved lazy propagation (improved LP), the latter of which optimizes computational efficiency through relevance reasoning and incremental updates. The experimental procedure is described as follows:

\begin{itemize}
	\item For each Bayesian network Alarm, Andes, Pigs, and Link, we randomly selected 100 sets of variables of interest $Q$ ($|Q| = 2$).
	
	\item 
	The CMDSA algorithm is used to find the d-convex hull $H$ that contains $Q$ and cache the conditional probability tables to construct the local model \(\mathcal{B}_H\), where \(\mathcal{B}\) represents the global model.

	\item We use the VE, LP, and improved LP algorithms to compute the marginal and conditional probabilities of Q based on the global model \(\mathcal{B}\). The conditional probability involves instantiating a variable in $Q$ and then computing the conditional probability of another variable. Additionally, we utilize the VE algorithm to compute the same probabilities based on the local model \(\mathcal{B}_H\), denoted as Con+VE.

    %We employed the VE, LP, and improved-LP algorithms to calculate the marginal probability and the conditional probability of $Q$ based on the global model \(\mathcal{B}\). This involves instantiating a variable in $Q$ and then computing the conditional probability of another variable. The VE algorithm was also used to compute the same probability of \(Q\) based on \(\mathcal{B}_H\). %$\mathrm{Con+VE}$ is the method used to compute marginals and conditional probability by applying the VE algorithm based on the localized model \(\mathcal{B}_H\).
	
	\item In each network \(\mathcal{B}\), the inference times for marginal and conditional probabilities are accumulated separately, and the total time is divided by 100 (the number of $Q$) to obtain the average computation time, which serves as the required time for the algorithm to compute marginal and conditional probabilities in \(\mathcal{B}\). Additionally, the Con+VE method also includes the time for finding the d-convex hull. %For each set of variables \(Q\) in each network, the inference was repeated 100 times, and the average runtime for these 100 inferences was calculated.
\end{itemize}

\begin{table}[htbp]
\centering
\caption{Times Required for Computing Conditional Probabilities Using Four Algorithms on Networks of Varying Scales}
\begin{tabular*}{\textwidth}{@{\extracolsep{\fill}}cccccc@{\extracolsep{\fill}}}
\hline
Network & Inference Times & \textbf{$T_{VE}$} & \textbf{$T_{BP}$} & \textbf{$T_{Improved}$} & \textbf{$T_{Our}$} \\ \hline
Alarm & 60624 & 77.9771 & 80.6846 & 62.5647 & \textbf{10.0977} \\ 
Hailfinder & 772416 & 1279.9336 & 1362.1269 & 750.4221 & \textbf{62.5342} \\ 
Hepar2 & 6688 & 16.5891 & 17.4625 & 11.2600 & \textbf{1.4639} \\ 
Win95pts & 2560 & 3.4633 & 3.6160 & 2.1355 & \textbf{0.6826} \\ 
Andes & 2560 & 27.1642 & 28.3393 & 12.2174 & \textbf{10.8166} \\ 
Pigs & 94041 & 152.4072 & 166.1403 & 56.5318 & \textbf{15.5688} \\ 
Link & 18688 & 304.8775 & 316.8837 & 127.0781 & \textbf{7.4304} \\ \hline
\end{tabular*}
\label{tab-1}
\end{table}

During marginal calculations, the improved LP algorithm could not optimize computations through relevance reasoning and incremental updates to enhance efficiency. Furthermore, in the Link network, the storage and computational demands of the lazy propagation algorithm increased significantly, resulting in reduced inference efficiency; thus, we excluded these results. The experimental results, as shown in Table \ref{tab-1}, indicate that in low-dimensional networks (Alarm and Andes), our method is slightly less efficient in conditional probability calculation due to the need to search for d-convex hulls. However, in high-dimensional networks (Pigs and Link), this dimensionality reduction preprocessing significantly improves the efficiency of conditional probability calculations. Notably, our method consistently achieves superior performance in marginal probability inference across networks of varying dimensions.

In addition, to evaluate the inference accuracy of the local models obtained through parameter learning, we designed the following experiments:
\begin{itemize} 
\item For $n \in [20, 40, \dots, 100]$, 10 directed trees are randomly generated, and edges are added to the tree with a probability $p \in [0.1, 0.2]$. The maximum number of parent nodes per node is also restricted to 3.

\item In each network, 5 sets of interest variables $Q$ (where $|Q| = 5$) are randomly selected, and the corresponding d-convex hull $H$ is searched. Then, local data is used to parameterize the d-convex hull, resulting in submodels $\mathcal{B}\prime_H$, with sample sizes of 1000 and 5000.

\item For each network $\mathcal{B}$, we compute the average KL divergence between the 5 sets of models $\mathcal{B}^\prime_H$ and $\mathcal{B}_H$, where the latter is inferred from model $\mathcal{B}$ and can be considered as the true model on $H$.

\item For each scale, the average KL divergence across the 10 Bayesian networks is computed.
\end{itemize}

\begin{table}[!t]
\caption{The average KL divergence from the true model $\mathcal{B}_H$ to $\mathcal{B}^\prime_H$ for different network sizes.}
\centering
\begin{tabular*}{\textwidth}{@{\extracolsep{\fill}}cccccc@{\extracolsep{\fill}}}
\toprule
\multirow{2}{*}{$n$} & \multirow{2}{*}{Sample size} & \multicolumn{4}{c}{Probability of adding edges} \\
\cmidrule{3-6}
 & & 0.1 & 0.05 & 0.01 & 0.005 \\
\midrule
\multirow{2}{*}{20} & 1000 & 0.0129 & 0.0191 & 0.0050 & 0.0099 \\
 & 5000 & 0.0033 & 0.0031 & 0.0011 & 0.0018 \\
\cline{1-6}
\multirow{2}{*}{40} & 1000 & 0.0156 & 0.0245 & 0.0119 & 0.0113 \\
 & 5000 & 0.0033 & 0.0068 & 0.0019 & 0.0008 \\
\cline{1-6}
\multirow{2}{*}{60} & 1000 & 0.0283 & 0.0213 & 0.0137 & 0.0133 \\
 & 5000 & 0.0049 & 0.0030 & 0.0031 & 0.0038 \\
\cline{1-6}
\multirow{2}{*}{80} & 1000 & 0.0119 & 0.0194 & 0.0162 & 0.0141 \\
 & 5000 & 0.0053 & 0.0053 & 0.0023 & 0.0045 \\
\cline{1-6}
\multirow{2}{*}{100} & 1000 & 0.0248 & 0.0246 & 0.0168 & 0.0115 \\
 & 5000 & 0.0046 & 0.0053 & 0.0040 & 0.0032 \\
\bottomrule
\end{tabular*}
\label{tab-4}
\end{table}

The experimental results are presented in Table \ref{tab-4}. The results show that as the sample size increases, the submodel $\mathcal{B}^\prime_H$ obtained through parameter learning increasingly approximates the true submodel. Therefore, when the true Bayesian network has a very high dimension, making it difficult to construct local models, the submodel $\mathcal{B}^\prime_H$ can be used for rapid inference, yielding almost exact results.

\section{Conclusion and Discussion}\label{sec-6}

This work focuses on providing a solution for structural dimension reduction in BNs, which can greatly improve the efficiency of probabilistic inference. We established the equivalence between d-convex subgraphs and M-collapsibility. This enables us to transform the problem of finding collapsible sets in graphical models into the problem of finding d-convex subgraphs, thereby converting an abstract statistical problem into an intuitive one in graph theory. Empirical studies support that the proposed mechanism performs excellently in most cases.

An important direction for future research is to leverage d-convex subgraphs in constructing decomposition trees of DAGs. While some researchers have developed d-separation trees for DAGs, these efforts primarily focused on improving the efficiency of structural learning rather than statistical inference \cite{xie2006decomposition}. In the future, we intend to develop a decomposition theory of DAGs to enhance the efficiency of estimation and testing in BNs.

{\bf Limitations}: When the network is highly complex or the selected set of variables of interest is large, the effectiveness of the CMDSA algorithm in dimensionality reduction may be limited. However, such networks are challenging for any exact inference algorithm. Moreover, if the set of variables of interest changes, it requires searching the d-convex hull again to establish a new local model.

%\backmatter
%\bmsection*{Author contributions}

%\bmsection*{Acknowledgments}

%\bmsection*{Financial disclosure}

%None reported.

%\bmsection*{Conflict of interest}

%The authors declare no potential conflict of interest.

%\bibliography{wileyNJD-AMA}

\appendix

\section*{Basic results}\label{new_N.0}

\begin{lemma}\label{new-lemma-1}
	If $\{u,v\}\in \mathcal{I}_G(R)$, then there exists a d-connected component $M\subseteq V\setminus R$ in $G$ w.r.t. $R$ such that $u,v\in mb_G(M)$.
\end{lemma}
\begin{proof}
	Let $l_{uv}$ be an inducing path of $R$ that connects $u$ and $v$ in $G$. Then $V^o(l_{uv})\cap R\subseteq V^c(l_{uv})$, implying  $V(l_{uv})\backslash R$ is d-connected w.r.t $R$ by d-connectedness. Hence there exists a d-connected component $M\subseteq V\backslash R$ in $G$ w.r.t. $R$ that contains $V(l_{uv})\backslash R$, indicating that $u,v\in mb_G(M)$.
\end{proof}

\begin{lemma}[Inheritance]\label{new-lemma-2}
	Let $G=(V,E)$ be a DAG and $R\subseteq V$ a d-convex subset of $G$. For $R^{\prime}\subseteq R$, we have that $R^{\prime}$ is d-convex in $G_R$ if and only if $R^{\prime}$ is d-convex in $G$.
\end{lemma}
\begin{proof}
	This follows from  $I(G_{R^{\prime}})=I(G_{R})_{R^{\prime}}=I(G)_{R^{\prime}}$ by Theorem~\ref{thm-3.5}.
\end{proof}

\begin{lemma}\cite{xie2009collapsibility}\label{lem:xie-geng}
	Given a BN $\mathcal{B}=(G,\mathcal{P}(G))$ with $G=(V,E)$ and $H\subseteq V$. Under faithfulness, the following properties are equivalent:
	\begin{itemize}
		\item[(i)] $\mathcal{P}(G)_H=\mathcal{P}(G_H)$;
		\item[(ii)] $I(G)_H=I(G_H)$.
	\end{itemize}    
\end{lemma}

\begin{definition}[Reachable set]\label{def:reachable-set}
	Given two non-adjacent vertices $u,v\in V$ in a DAG $G=(V,E)$, let $W\subseteq An_G(\{u,v\})$ and $U^\ast = \{z: \exists l_{uz}\subseteq G_{An_G(\{u,v\})}~ s.t.~ W \cap V^o(l_{uz})\subseteq V^c(l_{uz})\}$. We state that $U^\ast$ is the reachable set from $u$ given $W$ in $G$.
\end{definition}

\begin{proposition}\label{prop-5.6}
	Given a DAG $G=(V,E)$ and two non-adjacent vertices $u,v\in V$. The algorithm FCMDS has the complexity at most $O(|An_G(\{u,v\})|+|E_{An_G(\{u,v\})}|)$, where $|A|$ is the number of elements in $A$ and $E_{An_G(\{u,v\})}$  the edge set of the induced subgraph  $G_{An_G(\{u,v\})}$.
\end{proposition}

\begin{proof}
	In the worst case, consider that the complexity of searching the Markov boundary of $u$ is $O(|E_{An_G(\{u,v\})}|)$  and the reachable set can be found in $O(|An_G(\{u,v\})|+|E_{An_G(\{u,v\})}|)$ time by Bayes-ball algorithm \cite{shacter1998bayes}, from which it follows our result.
\end{proof}

\begin{algorithm}[h]
	\caption{Algorithm for Finding Close Minimal D-separators (FCMDS)}
	\label{alg-1}
	\begin{algorithmic}
		\STATE {\bfseries Input:} A DAG $G=(V,E)$  and two non-adjacent vertices $u,v\in V$.
		\STATE {\bfseries Output:} The minimal $uv$-d-separator $S$ close to $u$.
		\STATE Compute $mb_{G_{An_G(\{u,v\})}}(u)$; 
		\STATE Compute the reachable set $V^*$ of vertex $v$ given $mb_{G_{An_G(\{u,v\})}}(u)$.
		\STATE {\bfseries Return:} $mb_{G_{An_G(\{u,v\})}}(u)\cap V^*$
	\end{algorithmic}
\end{algorithm}

Let $S$ be the set returned from the algorithm FCMDS. For any $w\in S$, it can be easily observed that there exists a path $\rho_{uv}\subseteq (G_{An_G(\{u,v\})})^m$ connecting $u$ and $v$ such that $V^o(\rho_{uv}) \cap S = \{w\}$. Therefore, this implies that $S$ is a minimal separator in $(G_{An_G(\{u,v\})})^m$ and also a minimal d-separator in $G_{An_G(\{u,v\})}$.

%\subsection*{Subsection title of first appendix\label{app1.1a}}
%\bmsubsubsection{Subsection title of first appendix\label{app1.1.1a}}
%\noindent\textbf{Unnumbered figure}

\section*{The proof of the main conclusions}\label{new_N.b}
\vspace*{12pt}
\subsection*{Proof of Lemma \ref{thm-3.2}}\label{sec-A}
We say that three vertices $u,v,w$ constitute a v-structure if they induce the subgraph ``$u\rightarrow w\leftarrow v$'' in $G$ and $u,v$ are not adjacent.
\begin{proof}
	To show (i) $\Rightarrow$ (ii), suppose there is an inducing path $l_{uv}$ of $R$ between $u$ and $v$. We first prove that $v_i\in an_G(\{u,v\})$ for all $v_i\in V^o(l_{uv})$ by considering the following two cases:
	\begin{itemize}
		\item If $v_i\in V^{c}(l_{uv})$, it is evident that $v_i\in an_G(\{u,v\})$.
		
		\item If $v_i\in V^{s}(l_{uv})\cup V^{d}(l_{uv})$, let $V^{c}(l_{uv}) = \{c_0,\dots,c_k\}$. Then there exists $j$ ($0 \leq j\leq k-1$) such that $v_i\in V^{o}(l_{c_jc_{j+1}})$, where $l_{c_jc_{j+1}}\subseteq l_{uv}$. This gives that $v_i\in an_G(\{c_j,c_{j+1}\})\subseteq an_G(\{u,v\})$.
	\end{itemize}
	
	Therefore, one can see that $\rho_{uv}=(u,u_1,\dots,u_m,v)$ forms a path in $(G_{An_G(\{u,v\})})^{m}$, where $\{u_1,\dots,u_m\} =  V^{s}(l_{uv})\cup V^{d}(l_{uv})$ and $\rho_{uv}$ satisfies $V^{o}(\rho_{uv})\subseteq M$.
	
	The implication of (ii) $\Rightarrow$ (i) can be obtained by replacing, correspondingly, moral edges on $\rho_{uv}$ with {\bf v-structures} in $G_{An_G(\{u,v\})}$.
	
	We next prove (i) $\Rightarrow$ (iv). By Definition~\ref{def-1}, it follows that $V^{nc}(l_{uv})\cap (R\backslash \{u,v\}) = \emptyset$, which implies $(V^o(l_{uv})\setminus V^c(l_{uv}))\cap Z=\emptyset$ for any $Z\subseteq R\backslash \{u,v\}$. For further analysis, let $V^{c}(l_{uv})=\{c_0,\dots,c_k\}$. We show that $u  \ndep v\ |\ Z [G]$ for any $Z\subseteq R\backslash \{u,v\}$ by considering the following two cases:
	
	\begin{itemize}
		\item If $V^{c}(l_{uv})\subseteq An_G(Z)$, it is obvious that $u  \ndep v\ |\ Z [G]$;
		\item If there is at least one vertex $c_i\in V^c(l_{uv})$ is not in $An_G(Z)$ for some $i\in [k]:=\{0,1,\ldots,k\}$. Without loss of generality, we assume that $c_j$ is the vertex closest to $u$ but not in $ An_G(Z)$. If $c_j\in an_G(v)$, consider the path $l^{\prime}_{uv}=l_{uc_j}\cup l^{\prime}_{c_jv}$, where $l^{\prime}_{c_jv}$ is a directed path from $c_j$ to $v$, and $l_{uc_j}\subseteq l_{uv}$. Since $V^{nc}(l^{\prime}_{uv}) \cap Z=\emptyset$ and $V^c(l^{\prime}_{uv})\subseteq An_G(Z)$, it follows that $l^{\prime}_{uv}$ is a d-connected path between $u$ and $v$ given $Z$. On the other hand, if $c_j\in an_G(u)$, consider the path $l^{\prime\prime}_{uv}=l^{\prime}_{uc_j}\cup l_{c_jv}$, where $l^{\prime}_{uc_j}$ is a directed path from $u$ to $c_j$, and $l_{c_jv}\subseteq l_{uv}$. Then $V^{nc}(l^{\prime\prime}_{uv}) \cap Z=\emptyset$ and $V^c(l^{\prime\prime}_{uv})\subsetneq V^c(l_{uv})$, we perform a similar analysis for vertices in $V^c(l^{\prime\prime}_{uv})$. Since the number of collider vertices is finite, we can obtain a d-connected path between $u$ and $v$ given $Z$.
	\end{itemize}
	
	The implication of  (iv) $\Rightarrow$ (iii) follows easily by allowing $Z = an_G(\{u,v\})\backslash M$.
	
	In the end, we show that (iii) $\Rightarrow$ (ii). It follows from (iii) that $u$ and $v$ cannot be separated by $an_G(\{u,v\})\backslash M$ in $(G_{An_G(\{u,v\})})^{m}$, which implies the existence of a path $\rho_{uv}$ in $(G_{An_G(\{u,v\})})^{m}$ such that $V^o(\rho_{uv})\cap (an_G(\{u,v\})\backslash M) =\emptyset$. Since $V^o(\rho_{uv})\subseteq an_G(\{u,v\})$, it follows that $V^o(\rho_{uv})\subseteq M$. 
	
	This completes the proof. 
\end{proof}

\subsection*{Proof of Theorem \ref{thm-3.5}}\label{sec-N.1}

\begin{proof}
	To prove (i) $\Rightarrow$ (ii), it is obvious $I(G)_R\subseteq I(G_R)$. Let $M=V\setminus R$. From Lemma~\ref{thm-3.2} it follows that $x  \indep y\ |\ an_G(\{x,y\})\backslash M[G]$ for any two non-adjacent vertices $x,y\in R$, which yields the reversal implication by contradiction. 
	
	To prove (ii) $\Rightarrow$ (i), we assume that there are two non-adjacent vertices $x,y\in R$ which are connected by an inducing path of $R$. Consider that $x \indep y\ |\ an_{G_R}(\{x,y\}) [G_R]$ and $an_{G_R}(\{x,y\}) = an_G(\{x,y\}) \backslash M$, from which it follows that $x  \indep y\ |\ an_G(\{x,y\}) \backslash M [G]$, leading to a contradiction by Lemma~\ref{thm-3.2}.
\end{proof}

\subsection*{Proof of Theorem \ref{thm-3.7}}\label{sec-N.2}

\begin{proof}
	Without loss of generality, we prove it when $\Lambda = \{1,2\}$. Assume that $H_1\cap H_2 =H$. By Lemma~\ref{new-lemma-2}, it suffices to show that $H$ is d-convex in $G_{H_1}$.  By contradiction, suppose that $\{u,v\}\in \mathcal{I}_{G_{H_1}}(H)$ . Let $l_{uv}$ be an inducing path that connects $u$ and $v$ in $G_{H_1}$. Since $H_2\cap V^o(l_{uv}) = H\cap V^o(l_{uv}) \subseteq V^c(l_{uv})\subseteq an_G(\{u,v\})$, which implies that $\{u,v\}\in \mathcal{I}_{G}(H_2)$. This contradicts the fact that $H_2$ is d-convex in $G$.
\end{proof}

\subsection*{Proof of Theorem \ref{thm-4.3}}\label{sec-A.1}
\begin{proof} 
	To prove sufficiency, let's consider the case where $u\indep v\mid \emptyset[G]$. In this case, it is evident that $S=\emptyset \subseteq H$. Therefore, we only need to focus on the scenario where $u\ndep v\mid \emptyset[G]$.
	
	Assume for contradiction that there exists a minimal $uv$-d-separator $S$ in $G$ such that $S\nsubseteq H$. It is obvious that $S\subseteq An_G(\{u,v\})$. Now, if $u\ndep v\mid S\cap H[G_H]$, then there must exist a path $l_{uv}$ in $G_H$ such that $(V^{s}(l_{uv})\cup V^{d}(l_{uv}))\cap (S\cap H) = \emptyset$ and $V^{c}(l_{uv})\subseteq An_{G_H}(S\cap H)$. This implies that $(V^{s}(l_{uv})\cup V^{d}(l_{uv}))\cap S = \emptyset$ and $V^{c}(l_{uv})\subseteq An_G(S)$. However, this contradicts the fact that $u\indep v\mid S[G]$. Thus, it follows that $u\indep v\mid S\cap H[G_H]$.
	
	Furthermore, considering that $G_H$ is a d-convex subgraph of $G$, we can conclude that $u\indep v\mid S\cap H[G]$. This contradicts the assumption that $S$ is a minimal $uv$-d-separator in $G$. Hence, our assumption that such an $S$ exists must be false, proving the sufficiency of the statement.

	To prove the necessity of the statement, let's assume that $H$ is not d-convex. In this case, there exist non-adjacent vertices $u$ and $v$ in $H$ that are connected by an inducing path of $H$. By Lemma \ref{thm-3.2}, it is stated that any subset of $H\backslash \{u,v\}$ cannot d-separate $u$ and $v$ in $G$. However, this contradicts the fact that $H$ contains all minimal $uv$-d-separators. If $H$ truly contains all minimal $uv$-d-separators, then there exists at least one subset of $H\backslash \{u,v\}$ that d-separates $u$ and $v$ in $G$.
	
	Hence, this contradiction shows that if $H$ is not d-convex, it cannot contain all minimal $uv$-d-separators. 
\end{proof}
	
\subsection*{Proof of Theorem \ref{theo-1}}\label{sec-N.3}
\begin{proof}
	Let $G_C$ be a d-connected component of $G_{An_G(\{u,v\})}$ w.r.t. $mb_{G_{An_G(\{u,v\})}}(u)$ that contains $v$. Then $S:= mb_{G_{An_G(\{u,v\})}}(C)$, which satisfies $S\subseteq mb_{G_{An_G(\{u,v\})}}(u)$, is a $uv$-d-separator. Furthermore, $S$ is necessarily encompassed by every $uv$-d-separator contained within $mb_{G_{An_G(\{u,v\})}}(u)$. Thus, $S$ is the unique minimal $uv$-d-separator within $mb_{G_{An_G(\{u,v\})}}(u)$.
\end{proof}

\subsection*{Proof of Theorem \ref{thm-4.33}}\label{sec-A.2}

\begin{proof}
	Let $G_H$ represent the subgraph obtained through CMDSA, and let $G_{H^{\prime}}$ be the d-convex hull that contains $R$. By utilizing Theorem \ref{thm-4.3}, it can be observed that $H \subseteq H^{\prime}$. Moving forward, the goal is to establish the d-convexity of $G_H$ in $G$. 
	
	Since $H$ absorbs new vertices in each iteration and the vertices of $G$ are finite, the algorithm will always output a subset $H$ which is d-convex. If not, the algorithm continues to absorb new vertices until $H$ becomes a d-convex subset. In the worst-case scenario, $H=V$. Thus, it can be concluded that $G_H$ is indeed d-convex in $G$. Furthermore, since $G_{H^{\prime}}$ represents the d-convex hull containing $R$, we can derive that $G_H=G_{H^\prime}$.
\end{proof}

\subsection*{Proof of Proposition \ref{prop-5.8}}\label{sec-N.5}

\begin{proof}
	In the worst case, the complexity of searching for the close minimal d-separator is $O(|V|+|E|)$, as we need to traverse all the vertices and edges of the graph. Additionally, we may need to absorb at most $|V| - 2$ minimal d-separators, which adds to the overall complexity. This gives the proof.
\end{proof}

%\nocite{*}% Show all bib entries - both cited and uncited; comment this line to view only cited bib entries;

%\bmsection*{Author Biography}

%\begin{biography}{\includegraphics[width=76pt,height=76pt,draft]{empty}}{
%{\textbf{Author Name.} Please check with the journal's author guidelines whether author biographies are required. They are usually only included for review-type articles, and typically require photos and brief biographies for each author.}}
%\end{biography}

\end{document}